%% file: main.tex
\documentclass{mimosis}

\usepackage{shellesc}
\input{preamble}

\newsavebox{\abstractbox}
\newenvironment{abstract}
{
  \begin{lrbox}{0}
    \begin{minipage}{\textwidth}
      \begin{center}\normalfont\bfseries\abstractname
    \end{center}\quotation}
    {\endquotation
    \end{minipage}
  \end{lrbox}%
\global\setbox\abstractbox=\box0 }

\usepackage{etoolbox}
\makeatletter
\expandafter\patchcmd\csname\string\maketitle\endcsname
{\vskip\z@\@plus3fill}
{\vskip\z@\@plus2fill\box\abstractbox\vskip\z@\@plus1fill}
{}{}
\makeatother

\title{Introduction to the Analysis of Probabilistic Decision-Making Algorithms}
\author{
  \textbf{Agustinus Kristiadi} \\[-0.25em]
  {\normalsize Western University and Vector Institute, Canada} \\[-0.25em]
  {\normalsize\url{https://agustinus.kristia.de}}
}
\date{}

\begin{document}

\renewcommand{\bibname}{References}
\def\chmoniker{Chapter}

\frontmatter

\begin{abstract}
  Decision theories offer principled methods for making choices under various types of uncertainty.
  Algorithms that implement these theories have been successfully applied to a wide range of real-world problems, including materials and drug discovery.
  Indeed, they are desirable since they can adaptively gather information to make better decisions in the future, resulting in data-efficient workflows.
  In scientific discovery, where experiments are costly, these algorithms can thus significantly reduce the cost of experimentation.
  Theoretical analyses of these algorithms are crucial for understanding their behavior and providing valuable insights for developing next-generation algorithms.
  However, theoretical analyses in the literature are often inaccessible to non-experts.
  This monograph aims to provide an accessible, self-contained introduction to the theoretical analysis of commonly used probabilistic decision-making algorithms, including bandit algorithms, Bayesian optimization, and tree search algorithms.
  Only basic knowledge of probability theory and statistics, along with some elementary knowledge about Gaussian processes, is assumed.
\end{abstract}

\maketitle

\mainmatter

\include{sources/001_decision_making}

\include{sources/002_concentration}

\include{sources/003_bandit_frequentist}

\include{sources/005_gp}

\include{sources/006_bayesopt_discrete}

\include{sources/007_bayesopt_continuous}


\bookmarksetup{startatroot}
\backmatter

\chapter{Acknowledgments}
\addcontentsline{toc}{chapter}{Acknowledgments}
AK thanks Gustavo Sutter and Tristan Cinquin for their feedback and participation in the reading group based on this monograph.

{
  \small
  \bibliography{main}
  \bibliographystyle{plainnat}
}


\end{document}

%% file: preamble.tex
\usepackage{metalogo}
\usepackage[bottom=4cm]{geometry}

\usepackage{scrlayer-scrpage}
\addtokomafont{pageheadfoot}{\normalfont}
\clearpairofpagestyles
\lehead{\headmark}
\rohead{\headmark}
\ofoot[\pagemark]{\pagemark}

\usepackage{etoolbox}
\usepackage[T1]{fontenc}
\DeclareSIUnit\px{px}

\usepackage{epigraph}
\definecolor{linkcolor}{rgb}{0,0,0}
\usepackage[%
  colorlinks = true,
  citecolor  = linkcolor,
  linkcolor  = linkcolor,
  urlcolor   = linkcolor,
  unicode,
  linktocpage,
]{hyperref}

\usepackage{bookmark}
\usepackage{changepage}

\usepackage{natbib}
\setcitestyle{authoryear,round,citesep={;},aysep={,},yysep={;}}

\makeatletter
\let\alloc@latex\alloc@
\def\alloc@#1#2#3#4{\newcount}
\makeatother

\usepackage[capitalise]{cleveref}
\usepackage[shortlabels]{enumitem}
\setlist{itemsep=0.1em,topsep=0.5em,parsep=0pt,partopsep=0pt}

\usepackage{amsmath,amssymb,amsfonts,amsthm}
\usepackage[T1]{fontenc}

\usepackage{newtxtext}
\usepackage[scale=0.95]{newtxtt}

\usepackage{bm}
\usepackage{newtxmath}
\DeclareMathAlphabet{\mathcal}{OMS}{cmsy}{m}{n}
\DeclareMathAlphabet{\mathbb}{U}{msb}{m}{n}

\DeclareSymbolFont{cmsymbols}{OMS}{cmsy}{m}{n}
\DeclareSymbolFont{cmlargesymbols}{OMX}{cmex}{m}{n}
\SetSymbolFont{cmlargesymbols}{bold}{OMX}{cmex}{m}{n}

\let\sum\relax
\let\prod\relax
\let\int\relax
\let\oint\relax
\let\bigcup\relax
\let\bigcap\relax

\DeclareMathSymbol{\sum}{\mathop}{cmlargesymbols}{"50}
\DeclareMathSymbol{\prod}{\mathop}{cmlargesymbols}{"51}
\DeclareMathSymbol{\int}{\mathop}{cmlargesymbols}{"52}
\DeclareMathSymbol{\oint}{\mathop}{cmlargesymbols}{"48}
\DeclareMathSymbol{\bigcup}{\mathop}{cmlargesymbols}{"53}
\DeclareMathSymbol{\bigcap}{\mathop}{cmlargesymbols}{"54}

\makeatletter
\def\th@plain{%
  \thm@notefont{}
  \itshape 
}
\def\th@definition{%
  \thm@notefont{}
  \normalfont 
}
\makeatother

\usepackage[scale=.9]{plex-mono}

\usepackage[framemethod=TikZ]{mdframed}
\usepackage{thmtools}

\declaretheoremstyle[
  notefont=\bfseries,
  mdframed={
    linewidth=1pt,
    nobreak=true,
    innerbottommargin=1em,
    skipabove=1em,
    skipbelow=1em,
  }
]{boxedstyle}

\declaretheorem[style=boxedstyle, name=Theorem, numberwithin=chapter]{boxedtheorem}
\declaretheorem[name=Theorem, numberwithin=chapter, sibling=boxedtheorem]{theorem}
\declaretheorem[name=Propositon, numberwithin=chapter, sibling=theorem]{proposition}

\declaretheorem[name=Example, numberwithin=chapter, sibling=theorem]{example}
\declaretheorem[name=Remark, numberwithin=chapter, sibling=theorem]{remark}
\declaretheorem[style=boxedstyle, name=Remark, numberwithin=chapter, sibling=theorem]{boxedremark}

\input{math_commands}

\makeatletter
\def\th@plain{%
  \thm@notefont{}
  \itshape 
}
\def\th@definition{%
  \thm@notefont{}
  \normalfont 
}
\makeatother

\makeatletter
\def\thm@space@setup{%
  \thm@preskip=1em
  \thm@postskip=\thm@preskip 
}
\makeatother

\usepackage{cancel}
\usepackage{nicefrac}

\newcommand{\defword}[1]{\textbf{\emph{#1}}}

\newcommand{\rnum}[1]{%
  \textup{\uppercase\expandafter{\romannumeral#1}}%
}

\allowdisplaybreaks

\usepackage{algorithm}
\usepackage{algpseudocode}

\usepackage{pgfplots}
\pgfplotsset{compat=newest}
\usepgfplotslibrary{groupplots}
\usepgfplotslibrary{dateplot}

\pgfplotsset{compat=1.11,
  /pgfplots/ybar legend/.style={
    /pgfplots/legend image code/.code={%
      \draw[##1,/tikz/.cd,yshift=-0.25em]
    (0cm,0cm) rectangle (3pt,0.8em);},
  },
}

\usepackage{tikz-cd}
\tikzcdset{
  arrow style=tikz,
  diagrams={>={Straight Barb}}
}

\usetikzlibrary{calc}
\usetikzlibrary{positioning}
\usetikzlibrary{angles,quotes}
\usetikzlibrary{backgrounds}
\usetikzlibrary{fit}
\usetikzlibrary{arrows}
\usetikzlibrary{arrows.meta}
\usetikzlibrary{shapes.symbols}
\usetikzlibrary{shadings}
\usetikzlibrary{shapes}
\usetikzlibrary{patterns}
\usetikzlibrary{fadings}
\usetikzlibrary{bayesnet}
\usetikzlibrary{matrix}
\usetikzlibrary{plotmarks}
\usetikzlibrary{intersections}
\usetikzlibrary{pgfplots.fillbetween}

\usepackage[newfloat,frozencache=true,cachedir=minted-cache]{minted}
\setminted[python]{
  xleftmargin=12pt,
  mathescape,
  linenos,
  numbersep=5pt,
  frame=lines,
  framesep=2mm,
  fontsize=\scriptsize,
  highlightcolor=gray!15
}

\usepackage{tikzscale}

\colorlet{tablegray}{gray!15}

\definecolor{laplace1}{RGB}{38, 81, 206}
\definecolor{laplace2}{RGB}{232, 220, 87}
\definecolor{laplace3}{RGB}{238, 238, 238}

\definecolor{color0}{rgb}{0.00392156862745098,0.450980392156863,0.698039215686274}
\definecolor{color1}{rgb}{0.870588235294118,0.56078431372549,0.0196078431372549}
\definecolor{color2}{rgb}{0.00784313725490196,0.619607843137255,0.450980392156863}
\definecolor{color3}{rgb}{0.835294117647059,0.368627450980392,0}
\definecolor{color4}{rgb}{0.8,0.470588235294118,0.737254901960784}
\definecolor{color5}{rgb}{0.792156862745098,0.568627450980392,0.380392156862745}
\definecolor{color6}{rgb}{0.984313725490196,0.686274509803922,0.894117647058824}
\definecolor{color7}{rgb}{0.925490196078431,0.882352941176471,0.2}
\definecolor{color8}{rgb}{0.337254901960784,0.705882352941177,0.913725490196078}

\colorlet{gcolor0}{white}
\colorlet{gcolor1}{gray!15}
\colorlet{gcolor2}{gray!30}
\colorlet{gcolor3}{gray!45}
\colorlet{gcolor4}{gray!60}
\colorlet{gcolor5}{gray!75}
\colorlet{gcolor6}{gray!90}
\colorlet{gcolor7}{gray}
\colorlet{gcolor8}{black}

\definecolor{darkgray176}{RGB}{176,176,176}
\definecolor{darkslategray66}{RGB}{66,66,66}

\usepackage[titletoc,title]{appendix}
\AtBeginEnvironment{appendices}{\crefalias{section}{appendix}}

\usepackage{wrapfig}

\newcolumntype{H}{>{\setbox0=\hbox\bgroup}c<{\egroup}@{}}

\usepackage[all]{nowidow}

%% file: math_commands.tex
\def\mK{{\bm{K}}}

\DeclareMathAlphabet{\mathsfit}{\encodingdefault}{\sfdefault}{m}{sl}
\SetMathAlphabet{\mathsfit}{bold}{\encodingdefault}{\sfdefault}{bx}{n}

\def\gA{{\mathcal{A}}}

\def\gX{{\mathcal{X}}}

\newcommand{\E}{\mathbb{E}}

\newcommand{\R}{\mathbb{R}}

\newcommand{\Var}{\mathrm{Var}}

\DeclareMathOperator*{\argmax}{argmax}
\DeclareMathOperator*{\argmin}{argmin}

\newcommand{\N}{\mathcal{N}}

\renewcommand{\R}{\mathbb{R}}
\renewcommand{\E}{\mathrm{\mathbb{E}}}
\newcommand{\D}{\mathcal{D}}
\newcommand{\norm}[1]{\Vert #1 \Vert}
\newcommand{\abs}[1]{\vert #1 \vert}
\newcommand{\inv}{{-1}}

\newcommand{\inner}[1]{\langle #1 \rangle}

\newcommand{\X}{\mathcal{X}}

\newcommand{\I}{\mathbb{I}}
\renewcommand{\Pr}{\mathbb{P}}
\newcommand{\A}{\mathcal{A}}
\renewcommand{\O}{\mathcal{O}}
\newcommand{\GP}{\mathcal{GP}}
\renewcommand{\H}{\mathcal{H}}

%% file: sources/001_decision_making.tex

\chapter{Decision-Making}
\label{ch:decision}

\epigraph{
  ``\textit{Philosophically minded students of probability nimbly skip among these different ideas \textnormal{[frequentist and Bayesian]}, and take pains to say which probability concept they are employing at the moment. The vast majority of the practitioners of probability do no such thing. They go on talking of probability, doing their statistics and their decision theory oblivious to all this accumulated subtlety. [...] Extremists of one school or another argue vigorously that the distinction is a sham, for there is only one kind of probability.}''}{\citet[pp.~14]{hacking2006emergence}}

Humans make decisions constantly: ``What to eat for dinner?'', ``Which university to attend?'', ``What is a good rule-of-thumb when arriving in a foreign place?'', etc.
Artificial intelligence (AI) systems can also benefit from human-like decision-making.

To formulate decision-making processes, decision theory has been developed \citep{wald1949decision}.
Let \( \D \) be \defword{data} and \( f \) a \defword{latent variable/parameter}, generated under an \emph{unknown} joint distribution \( p(\D, f) \).
Furthermore, let \( a \in \gA \) be the set of all possible \defword{actions} that can be taken, and let \( u(f, a) \) be a \defword{utility function} that measures the ``compatibility'' of \( a \) and \( f \).\footnote{An alternative view, common in the frequentist formalism, is to replace the utility function with a \emph{loss function}, which, w.l.o.g., can be taken as \( -u \).}

\begin{example} \label{ex:decision:vars}
  \( \D \) could be a list of symptoms and \( f \) a disease.
  The set of \( \gA \) contains possible drugs that can be taken.
  The utility function \( u \) indicates the effectiveness of a drug against a disease.
\end{example}

There are two, non-mutually-exclusive ways to view statistical decision-making: Bayesian and frequentist.
As indicated in the epigraph of this chapter, in this text, we accept that they are both valid and useful for different purposes.

\section{Bayesian decision theory}
\label{sec:decision-making:bayesian}

Bayesian statistics assumes that probabilities represent \emph{beliefs}: ``I am 50\% sure'', ``I am quite confident that this will work'', etc.
Decision-making under this framework is then used to take an action based on one's (e.g., a model's) belief about the unknown.
Indeed, \defword{Bayesian decision theory} assumes that \( f \) is \emph{unknown} and \( \D \) is \emph{observed}, and concerns in finding the best action \( a_* \) under the \emph{posterior belief} \( p(f \mid \D) \) and the utility function \( u \):
\begin{equation}
  a_* = \argmax_{a \in \gA} \E_{f \sim p(f \mid \D)}[u(f, a)] .
\end{equation}

\begin{example}
  Under the setting of \cref{ex:decision:vars}, a doctor can assess their belief \( p(f \mid \D) \) of a patient having a disease \( f \) after observing the patient's symptoms \( \D \).
  The doctor's utility function \( u \) encodes their preferences, e.g.\ whether they are rather risk-averse or risk-taking.
  The doctor then prescribes a drug \( a_* \) that maximizes their expected utility under posterior belief.
\end{example}

\begin{example}[Pascal's Wager]
  Let \( \gA = \{ \text{``Believe in God}, \text{``Don't believe in God''} \} \) and let \( f \in \{ \text{``God exists''}, \text{``God doesn't exist''} \} \).
  Suppose our utility weighs the potential \textbf{eternal} ``reward'' or ``punishment'' in heaven and hell, respectively.
  It makes sense, therefore, for someone to have the utility function \( u(f, a) \) s.t.:
  \begin{itemize}
    \item \( u(\text{``God exists''}, \text{``Believe in God''}) = \infty \),
    \item \( u(\text{``God doesn't exist''}, \text{``Believe in God''}) = a \) where \( -\infty < a < 0 \); notice that while this is undesirable (negative utility), \( a \) is finite,
    \item \( u(\text{``God exists''}, \text{``Don't believe in God''}) = -\infty \),
    \item \( u(\text{``God doesn't exist''}, \text{``Don't believe in God''}) = b \) where \( 0 < b < \infty \); where the argument here is that we gain something that is finite in our lifetime, but nothing more.
  \end{itemize}
  In this case, the optimal action \( a_* \) is to ``believe in God'', even if \emph{a posteriori}, \( p(\text{``God exists''} \mid \D) \) is very small.
  Do note that different utility functions will yield different \( a_* \).
\end{example}

\section{Frequentist decision theory}
\label{sec:decision-making:frequentist}

Frequentist statistics assumes that probabilities represent \emph{long-running relative frequency}: ``What is the occurrence of a disease in a given population?'', ``The error rate of this program is 1\%'', etc.
That is, if we sample the data again and again, what is the proportion of a case of interest?
Notice that, here, the data is therefore assumed to be random.

Thus, in contrast to the previous section, \defword{frequentist decision theory} assumes that \( f \), while still unknown, is fixed and \( \D \) is generated through \( \D \sim p(\D \mid f) \).
Notice that, here, the roles of \( \D \) and \( f \) are reversed compared to their roles in the Bayesian decision theory.
Since \( \D \) is now random, we aim to find an optimal \emph{function} \( \delta_* \) that maps a realization of data to an action:
\begin{equation}
  \delta_* = \argmax_{\delta} \E_{\D \sim p(\D \mid f)}[u(f, \delta(\D))] .
\end{equation}
or, equivalently,
\begin{equation}
  \delta_* = \argmin_{\delta} \E_{\D \sim p(\D \mid f)}[\ell(f, \delta(\D))] ,
\end{equation}
where \( \ell = -u \) is the so-called \defword{loss function}.
That is, we want to find the best ``policy'' \( \delta_* \) that works in various situations \( \D \sim p(\D \mid f) \).
The function \( u \) is interpreted as measuring how good it is to do an action \( \delta(\D) \) under the data \( \D \) given that \( f \) is the underlying parameter that generates \( \D \).

\begin{example}
  Under the setting of \cref{ex:decision:vars}, suppose a public health organization wants to recommend a policy/rule---a ``what-to-do'' guideline---for the general population.
  The goal is then to find a policy \( \delta_* \) such that when presented with a set of symptoms \( \D \), it recommends a drug \( a \) for treating the underlying, unknown disease \( f \).
\end{example}

However, notice that \( u \) depends on \( f \), which we have assumed to be unknown.
It follows that we cannot even compute \( u(f, \delta(\D)) \) and thus we cannot perform the maximization.
We do not have such a problem in the Bayesian case since we have a belief about \( f \).

To circumvent this issue, we need to take \( f \) out of the equation.
One way to do so is as follows.
Let \( R(f, \delta) = \E_{\D \sim p(\D \mid f)} [\ell(f, \delta(\D))] \) be the \defword{risk} function.
Then, we find the optimal decision function \( \delta_* \) that minimizes the worst-case risk:
\begin{equation}
  \delta_* = \argmin_{\delta} \max_{f} R(f, \delta) .
\end{equation}
Continuing the previous example, the minimax decision function has the interpretation that it is the one that is optimal under the worst-case risk when we consider various plausible alternatives of the underlying diseases \( f \).

\begin{example}
  If we think \( f \) could be ``cold'', ``malaria'', and ``COVID-19'', then we want to provide a treatment guideline that would be relatively effective for every possible \( f \).
  Note that \( \delta_* \) might not be the best possible guideline for each individual \( f \).
\end{example}

\begin{boxedremark}
  Should you be Bayesian or frequentist?
  \emph{Both!}
  Hopefully, the epigraph and the discussion in this chapter convinced you that being both is the correct move.
  They are different tools for different situations and goals.
\end{boxedremark}

%% file: sources/002_concentration.tex

\chapter{Concentration Inequalities}
\label{ch:concineq}

Let \( X \) be a random variable (r.v.).
It is often useful to know how such a r.v.\ \emph{concentrated} around a value.
\defword{Concentration inequalities} are the bread-and-butter for theoretical analyses in machine learning (and other fields!).

\begin{example}
  Let \( X \) indicate the grade of a random student in a given population (e.g.\ in a class).
  A concentration inequality is useful to answer the following question:
  ``What is the prevalence of students with a grade at least (e.g.) \( 80 \)?''
\end{example}

Let \( \Pr \) denote the long-running relative frequency of its random \emph{event} argument.
That is, let \( (X_i)_{i=1}^n \) be a sequence of samples of \( X \) and \( A \) be an event that depends on \( X \) such as \( X \geq a \) for some value \( a \), then we can write \( \Pr(A) = \lim_{n \to \infty} \frac{1}{n} \sum_{i=1}^n \I[A_i] \),
where \( \I[\, \cdot \,] \) denotes the indicator function, which equals one if its event argument happens, and zero otherwise.
We call \( \Pr(A) \) the \defword{probability} of observing the event \( A \).

\begin{example} \label{ex:weight-lt}
  Let \( X \) indicate the grade of a random student in a given population.
  Let \( Z \) be a random event indicating \( X \geq 80 \), i.e.\ the event when a random student has a grade greater than or equal to \( 80 \).
  Then \( \Pr(Z) \) indicates the proportion of students with grade \( \geq 80 \) when we pick and measure people at random many times.
\end{example}

\begin{remark}
  The notion of probability here differs from Bayesian statistics, where it denotes \emph{degree of belief} about some event.
\end{remark}

Here, we will describe some useful concentration inequalities.
The proofs are omitted since they are standard, and we focus on their applications.

\begin{boxedtheorem}[Markov's Inequality]
  Let \( X \) be a nonnegative r.v.\ and assume that \( \E(X) \), the expected value of \( X \) exists (i.e., \(\E(X) < \infty \)).
  Then,
  \begin{equation}
    \label{eq:markov1}
    \Pr(X \geq a) \leq \frac{\E(X)}{a}
  \end{equation}
  for all \( a > 0 \).
  Moreover,
  \begin{equation}
    \label{eq:markov2}
    \Pr(X \geq b \E(X)) \leq \frac{1}{b}
  \end{equation}
  for all \( b > 1 \).
\end{boxedtheorem}

\textbf{Markov's inequality} is useful if we only know the expected value \( \E(X) \) of \( X \).

\begin{example} \label{ex:markov}
  Continuing \cref{ex:weight-lt}, suppose we know that on average, students have a grade \(50\).
  Then, the proportion of people in the population weighing \( \geq 80 \) is at most: \( \Pr(X \geq 80) \leq \frac{50}{80} = 0.625 \).
\end{example}

However, notice that the bound on the probability decreases \emph{linearly} in \( a \).
We can obtain a tighter bound if we know further information about \( X \), namely its \emph{variance} \( \Var(X) \).

\begin{boxedtheorem}[Chebyshev's Inequality]
  Let \( X \) be a r.v.\ and assume that both \( \E(X) \) and \( \Var(X) \) exist.
  Then,
  \begin{equation}
    \label{eq:chebyshev1}
    \Pr(\abs{X - \E(X)} \geq a) \leq \frac{\Var(X)}{a^2}
  \end{equation}
  for all \( a > 0 \).
  Moreover,
  \begin{equation}
    \label{eq:chebyshev2}
    \Pr(\abs{X - \E(X)} \geq b \abs{\E(X)}) \leq \frac{\Var(X)}{b^2 (\E(X))^2}
  \end{equation}
  for all \( b > 1 \).
\end{boxedtheorem}

Chebyshev's inequality is useful to bound the proportion of some measurements deviating from the population's mean.

\begin{example} \label{ex:chebyshev}
  Let \( X \) and \( \E(X) \) be as in \cref{ex:markov}.
  Suppose \( \Var(X) = 10 \).
  Then,
  \begin{equation*}
    \Pr(\abs{X - 50} \geq 30) \leq \frac{10}{30^2} = 0.0111 .
  \end{equation*}
  That is, it is quite rare that a random student's grade deviates by \( 30 \) or more from the average in a class with \( \E(X) = 50 \) and \( \Var(X) = 10 \).
\end{example}

Recall that Chebyshev's inequality is tighter than Markov's inequality, and they differ in how they use the moments (mean, variance) of the r.v.
It is thus logical to ask:
Can we get a tighter bound if we consider higher moments?
The answer is the Chernoff bound.
First, let us state its special version for Bernoulli random variables.

\begin{boxedtheorem}[Chernoff Bound---Bernoulli]
  Let \( (X_i)_{i=1}^n \) be independent Bernoulli r.v.s.\ with their respective expectations \( (p_i)_{i=1}^n \).
  Let \( X = \sum_{i=1}^n X_i \).
  Then,
  \begin{enumerate}[(i)]
    \item For every \( 0 < \delta \leq 1 \), it holds that
          \begin{equation} \label{eq:chernoff-bern-1}
            \Pr(X \geq (1 + \delta) \E(X)) \leq \exp(-\E(X) \delta^2 / 3) .
          \end{equation}
    \item For every \( 0 < \delta < 1 \), it holds that
          \begin{equation} \label{eq:chernoff-bern-2}
            \Pr(X \leq (1 - \delta) \E(X)) \leq \exp(-\E(X) \delta^2 / 2) .
          \end{equation}
  \end{enumerate}
\end{boxedtheorem}

Notice that the bound of (this version) of the Chernoff bound is exponential in the constant \( a \).
Let us compare the ``strength'' of Markov's, Chebyshev's, and Chernoff's bounds in the following example.

\begin{example} \label{ex:compare-markov-cheb-chern}
  Consider \( (X_i)_{i=1}^n \) be \( n \) independent tosses of a fair coin---\( X_i = 1 \) if head and \( X_i = 0 \) otherwise.
  Let \( X = \sum_{i=1}^n X_i \) denote the number of heads we see.
  It's expected value is thus \( \E(X) = \frac{1}{n} \sum_{i=1}^n \E(X_i) = \frac{n}{2} \).
  We would like to see the frequency of the event where the number of heads \( \geq \frac{3}{4} n \).
  With Markov's inequality, we see that
  \begin{equation*}
    \Pr \left( X \geq \frac{3}{4}n \right) \leq \frac{n/2}{(3/4)n} = \frac{2}{3} .
  \end{equation*}
  Notice the constant bound.
  Meanwhile, with Chebyshev's inequality, we obtain
  \begin{equation*}
    \Pr \left( X \geq \frac{3}{4}n \right) \leq \Pr \left( \left\vert X - \frac{n}{2} \right\vert \geq \frac{n}{4} \right) \leq \frac{\Var(X)}{(\frac{n}{4})^2} = \frac{n/4}{n^2/16} = \frac{4}{n} .
  \end{equation*}
  by noting that \( \Var(X) = \frac{n}{4} \).
  This bound is indeed stronger than Markov's since it decreases as \( n \) increases.
  Finally, for the Chernoff bound, we let \( \delta = 1/2 \) since then \( (1 + \delta) \E(X) = \frac{3}{4} n \), and obtain
  \begin{equation*}
    \Pr \left( X \geq \frac{3}{4}n \right) \leq \exp(-\E(X) \delta^2 / 3) = \exp\left( -\frac{n}{2} \frac{1}{4} \frac{1}{3} \right) = \exp(-n/24) .
  \end{equation*}
  Notice that the Chernoff bound decreases exponentially in the number of tosses.
\end{example}

The following is the general version of the Chernoff bound.
Recall that \( M(t) = \E(\exp(tX)) \) is the moment-generating function of \( X \).

\begin{boxedtheorem}[Chernoff Bound]
  Let \( X \) be a random variable and let \( a \) be an arbitrary value of \( X \).
  Then,
  \begin{enumerate}[(i)]
    \item For every \( t > 0 \), it holds that
          \begin{equation} \label{eq:chernoff-1}
            \Pr(X \geq a) \leq \E(\exp(tX)) \exp(-ta) .
          \end{equation}
    \item For every \( t < 0 \), it holds that
          \begin{equation} \label{eq:chernoff-2}
            \Pr(X \leq a) \leq \E(\exp(tX)) \exp(-ta) .
          \end{equation}
  \end{enumerate}
\end{boxedtheorem}

Next, we have a similar, exponentially decreasing bound in the form of \defword{Hoeffding's inequality} which requires us to know the upper and lower bounds of the values of the r.v.s.

\begin{boxedtheorem}[Hoeffding's Inequality] \label{thm:hoeffding}
  Let \( (X_i)_{i=1}^n \) be i.i.d.\ r.v.s.\ with mean \( \mu \), where for each \( i \) we have \( l \leq X_i \leq h \).
  Let \( \bar{X} = \frac{1}{n} \sum_{i=1}^n X_i \) be their sample mean.
  Then,
  \begin{equation} \label{eq:hoeffding-1}
    \Pr\left(\abs{\bar{X} - \mu} \geq a\right) \leq 2 \exp\left( - \frac{2 n a^2}{(h - l)^2} \right) ,
  \end{equation}
  for all \( a > 0 \).
\end{boxedtheorem}

\section{Gaussian Tail Bounds}
\label{sec:background:gaussian-tail-bounds}

For Gaussian random variables, we have the following theorem \citep{srinivas2010gpucb}:

\begin{boxedtheorem}[Gaussian Tail Bound] \label{thm:gaussian-tail-srinivas}
  Let \( X \) be a Gaussian r.v.\ with mean \( \mu \) and variance \( \sigma^2 \).
  For any \( \beta > 0 \),
  \begin{equation}
    \Pr(\abs{X - \mu} \geq \beta \sigma) \leq \exp\left( -\beta^2 / 2 \right) .
  \end{equation}
\end{boxedtheorem}

Here is another useful property for Gaussian r.v.s.\ with nonpositive means:

\begin{boxedtheorem}[Gaussian Tail with Nonpositive Mean] \label{thm:gaussian-tail-nonpositive}
  Let \( X \) be a Gaussian r.v.\ with mean \( \mu \leq 0 \) and variance \( \sigma^2 \).
  Then,
  \begin{equation}
    \E(X \, \I(X \geq 0)) = \frac{\sigma}{\sqrt{2 \pi}} \exp \left( \frac{-\mu^2}{2 \sigma^2} \right) .
  \end{equation}
\end{boxedtheorem}

The expression \( X \I(X \geq 0) \) means we are looking at the Gaussian r.v.\ \( X \) where it takes values \( \geq 0 \) and ignore everywhere else.

\section{Other Useful Inequalities}
\label{sec:background:otherineqs}

In theoretical analysis, concentration inequalities are often paired with other inequalities.
Here, we shall see some of the commonly used inequalities.
The simplest is the \defword{union bound}.

\begin{boxedtheorem}[Union Bound] \label{thm:union-bound}
  Let \( (A_i)_{i=1}^n \) be a sequence of random events.
  Then,
  \begin{equation} \label{eq:union-bound}
    \Pr\left( \bigcup_{i=1}^n A_i \right) \leq \sum_{i=1}^n \Pr(A_i) .
  \end{equation}
\end{boxedtheorem}

\begin{example}
  Suppose the probability of a student getting a perfect \( 100 \) grade is at most \( 0.001 \).
  Denote \( A_i \) to be the event a student \( i \) gets grade \( 100 \).
  Then the probability of at least one student obtaining the perfect grade in a class of size \( 50 \) is
  \begin{equation*}
    \Pr\left( \bigcup_{i=1}^{12} A_i \right) \leq \sum_{i=1}^{50} \Pr(A_i) \leq \sum_{i=1}^{50} 0.001 = 0.05.
  \end{equation*}
  That is, there is at most \( 5\% \) chance/relative frequency that a student will get \( 100 \) in this setting.
\end{example}

Another useful inequality is \defword{Jensen's inequality} which allows us to swap an expectation operator with a convex/concave function.

\begin{boxedtheorem}[Jensen's Inequality] \label{thm:jensen}
  Let \( X \) be a random variable taking values in \( \R^n \) and let \( f: \R^n \to \R \) be a convex or concave function.
  Then,
  \begin{enumerate}[(i)]
    \item if \( f \) is \textbf{convex}: \( f(\E(X)) \leq \E(f(X)) \),
    \item if \( f \)  is \textbf{concave}: \( f(\E(X)) \geq \E(f(X)) \).
  \end{enumerate}
  Moreover, both inequalities also hold for empirical means.
\end{boxedtheorem}

%% file: sources/003_bandit_frequentist.tex

\chapter{Frequentist Bandits}
\label{ch:bandit_frequentist}

In \(K\)-armed bandit problem, we have \(K\) different actions \( a_t \in \A := \{ 1, \dots, K \} \) we can perform at each time step \( t = 1, \dots, T \).
After performing an action \( a \in \A \), the we observe a reward value \( r(a) \in [0, 1] \) distributed as an \emph{unknown} reward distribution \( p(r \mid a) \).

Let \( \mu(a) = \E(r(a)) \) be the unknown expected reward of action \( a \).
Let us also denote \( a_* = \argmax_{a \in \A} \mu(a) \) to be the action with the highest expected reward.
We can define
\begin{equation} \label{eq:bandits_frequentist:regret}
  R_T = \sum_{t=1}^T r(a_*) - r(a_t) ,
\end{equation}
called the \defword{regret} over a run of an algorithm where we select a sequence of actions \( (a_t)_{t=1}^T \).
This measures ``how far away'' our actions deviate from the optimal actions.

Since each \( a_t \) in \eqref{eq:bandits_frequentist:regret} is a random variable that depends on an algorithm's run, \( R_T \) is also a r.v.
Thus, it makes sense to study the \defword{expected regret}
\begin{equation} \label{eq:bandits_frequentist:exp-regret}
  \E(R_T) = \sum_{t=1}^T \mu(a_*) - \mu(a_t) = T \mu(a_*) - \sum_{t=1}^T \mu(a_t) .
\end{equation}
Ideally, an algorithm has \defword{no regret}, i.e., \( \lim_{T \to \infty} \E(R_T)/T = 0 \).
Our goal is to construct an algorithm for picking sequences of actions that minimize the expected regret and asymptotically have no regret.
The algorithm shall leverage \emph{frequentist} technique, e.g.\ using the sample mean to estimate \( \mu \) and making a decision based on this estimate.

\section{Explore-Then-Exploit}
\label{sec:bandits:explore-exploit}

The simplest algorithm is to explore for \( NK < T \) rounds and exploit for the remaining \( T - N \) rounds \citep{lattimore2020bandit}.
Exploration here means that we try each action \( N \) times.
Meanwhile, exploitation means that we use our estimate of the expected reward of each action, \( \hat\mu(a) = 1/N \sum_{t=1}^{N} r_t(a) \), to pick our estimate of best action \( \hat{a}_* = \argmax_{a \in \A} \hat\mu(a) \), and always pick this action.
The algorithm is summarized in \cref{alg:bandits_frequentist:ete}.

\begin{algorithm}
  \small

  \caption{Explore-Then-Exploit}
  \label{alg:bandits_frequentist:ete}

  \begin{algorithmic}[1]
    \Require{Time horizon \( T \), set of \( K \) actions \( \gA \), number of tries per action \( N \).}
    \Ensure{Cumulative reward \( r_\text{total} \)}

    \State \( r_\text{total} = 0 \)

    \For{\( t = 1, \dots, N \)}
    \ForAll{\( a \in \gA \)}
    \State \( r_{ta} = \texttt{do\_action}(a) \)
    \State \( r_\text{total} = r_\text{total} + r_{ta} \)
    \EndFor

    \State \( \hat{\mu}(a) = \frac{1}{N} \sum_{t=1}^N r_{ta} \) for each \( a \in \gA \)
    \State \( \hat{a}_* = \argmax_{a \in \gA} \hat{\mu}(a) \)

    \For{\( t = 1, \dots, T - NK \)}
    \State \( r_{ta} = \texttt{do\_action}(\hat{a}_*) \)
    \State \( r_\text{total} = r_\text{total} + r_{ta} \)
    \EndFor

    \EndFor
    \State \Return \( r_\text{total} \)
  \end{algorithmic}
\end{algorithm}

\begin{theorem} \label{thm:explore-exploit}
  With \( N = (T/K)^{2/3} (\log T)^{1/3} \), the explore-then-exploit algorithm has regret of
  \begin{equation*}
    \E(R_T)  \leq  \O\left((K T^2 \log T)^{1/3}\right) \qquad \text{with probability} \geq 1 - \frac{2K}{T^4} .
  \end{equation*}
  That is, it has no regret as \( T \to \infty \) with high probability.
\end{theorem}
\begin{proof}
  Let \( \hat\mu(a) = 1/N \sum_{t=1}^{N} r_t(a) \) be empirical average reward of action \( a \).
  Define \( \varepsilon = \sqrt{(2 \log T)/N} \).
  Also, define an event \( E = \{ \abs{\hat\mu(a) - \mu(a)} \leq \varepsilon; \forall a \in \A \} \) where all actions' estimates are within \( \varepsilon \) distance to the respective true values.

  Assume that \( E \) holds.
  Recall that \( \hat{a}_* = \argmax_{a \in \A} \hat\mu(a) \) and \( a_* = \argmax_{a \in \A} \mu(a) \).
  So, by definition,
  \begin{equation*}
    \hat\mu(\hat{a}_*) \geq \hat\mu(a_*) \quad\text{and}\quad  \mu(\hat{a}_*) \leq \mu(a_*) .
  \end{equation*}
  Now, since \( E \) holds, \( \hat\mu(\hat{a}_*) - \mu(\hat{a}_*) \leq \varepsilon \) and \(  \mu(a_*) - \hat\mu(a_*) \leq \varepsilon \).
  (Notice the absolute value in \( \E \).)
  And so,
  \begin{equation*}
    \mu(\hat{a}_*) + \varepsilon \geq \hat\mu(\hat{a}_*) \quad\text{and}\quad \hat\mu(a_*) \geq \mu(a_*) - \varepsilon .
  \end{equation*}
  Altogether they imply
  \begin{equation*}
    \begin{aligned}
       & \mu(\hat{a}_*) + \varepsilon \geq \hat\mu(\hat{a}_*) \geq \hat\mu(a_*) \geq \mu(a_*) - \varepsilon \\
       & \iff \mu(\hat{a}_*) + \varepsilon \geq \mu(a_*) - \varepsilon                                      \\
       & \iff 2 \varepsilon \geq \mu(a_*) - \mu(\hat{a}_*) .
    \end{aligned}
  \end{equation*}
  Hence, we have \( \mu(a_*) - \mu(\hat{a}_*) \leq 2 \sqrt{(2 \log T)/N} \).
  This is the bound on the regret during the exploitation phase, assuming that \( E \) holds.
  The upper bound on the regret during the exploration phase is trivially \( NK \) since \( \mu(\,\cdot\,) \in [0, 1] \).
  Thus, under \( E \),
  \begin{equation*}
    \E(R_T) \leq NK + \sum_{t=1}^{T-NK} 2 \sqrt{\frac{2 \log T}{N}} = NK + 2 (T-NK) \sqrt{\frac{2 \log T}{N}}.
  \end{equation*}
  Substituting \( N = (T/K)^{2/3} (\log T)^{1/3} \), we obtain \( \E(R_T) \leq \O((K T^2 \log T)^{1/3}) \).
  It is clear that \( \lim_{T \to \infty} \E(R_T) / T = 0 \) since \( (T^2 \log T)^{1/3}/T = (\log T)/T \).

  Now we compute the probability of the event \( E \).
  Using Hoeffding's inequality (\cref{thm:hoeffding}), we can bound the deviation \( \abs{\hat\mu(a) - \mu(a)} \) of our estimate to the true expected reward of action \( a \):
  \begin{equation*}
    \Pr(\abs{\hat\mu(a) - \mu(a)} \geq \varepsilon) \leq 2 \exp\left( - 2 N \varepsilon^2 \right) = \frac{2}{T^4} .
  \end{equation*}
  The complement of \( E \) is \( E^c = \{ \abs{\hat\mu(a) - \mu(a)} \geq \varepsilon; \exists a \in \A \} \).
  By the union bound (\cref{thm:union-bound}), we have
  \begin{equation*}
    \Pr(E^c) = \Pr \left( \bigcup_{a=1}^K \{ \abs{\hat\mu(a) - \mu(a)} \geq \varepsilon \} \right) \leq \frac{2K}{T^4} .
  \end{equation*}
  So, \( \Pr(E) = 1 - \Pr(E^c) \geq 1 - \frac{2K}{T^4} \).
  This is the probability of attaining the regret bound below.
  Note that we can ignore the event \( E^c \) since it occurs with such a low probability.
\end{proof}

\begin{remark}
  In summary, the proof strategy boils down to
  \begin{enumerate}
    \item defining a event \( E \) that encompass ``nice'' properties for our analysis,
    \item bounding the expected regret \( \mu(a_*) - \mu(a_t) \) at each time step \( t \) under \( E \),
    \item extending it to the bound of the cumulative expected regret \( \E(R_T) \) by summing them,
    \item reasoning about its expected value using the union bound and concentration inequality,
    \item arguing that the probability of the event \( E \) is high.
  \end{enumerate}
  To get the value for \( N \) in the hypothesis, one can aim to solve for the bound w.r.t.\ \( N \) s.t.\ the bound is minimized (i.e.\ tighter).
  If we only care about showing the no-regret property, we can pick \( N \) such that the bound is \emph{sublinear} in \( T \).
  Because, then, \( \E(R_T) \) will grow slower than \( T \) and thus \( \E(R_T) / T \) will converge to \( 0 \) \( \implies \) no regret.
\end{remark}

\section{Upper Confidence Bound (UCB)}
\label{sec:bandits:ucb}

Let us now consider the following decision rule \citep{auer2002ucb}:
At each time \( t = 1, \dots, T \), we pick an action that maximizes the function
\begin{equation} \label{eq:bandits:ucb}
  \mathrm{UCB}_t(a) = \mu_t(a) + \sqrt{(2 \log T) / N_t(a)} ,
\end{equation}
where \( \mu_t(a) = 1/N_t(a) \sum_{i=1}^{t} r_i(a_i) \I(a_i = a) \) is the empirical mean estimate of \( \mu(a) \) after \( t \) rounds,
and \( N_t(a) = \sum_{i=1}^{t} \I(a_t = a) \) is the number of times the action \( a \) has been selected.
The algorithm is summarized in \cref{alg:bandits_frequentist:ucb}.

\begin{algorithm}
  \small

  \caption{UCB}
  \label{alg:bandits_frequentist:ucb}

  \begin{algorithmic}[1]
    \Require{Time horizon \( T \), set of \( K \) actions \( \gA \)}
    \Ensure{Cumulative reward \( r_\text{total} \)}

    \State \( r_\text{total} = 0 \)

    \For{\( t = 1, \dots, T \)}
    \State Count \( N_t(a) \) for each \( a \in \gA \)
    \State Compute \( \mu_t(a) \) for each \( a \in \gA \)
    \State \( a_t = \argmax_{a \in \gA} \mu_t(a) + \sqrt{(2 \log T) / N_t(a)} \)
    \State \( r_{ta} = \texttt{do\_action}(a_t) \)
    \State \( r_\text{total} = r_\text{total} + r_{ta} \)
    \EndFor

    \State \Return \( r_\text{total} \)
  \end{algorithmic}
\end{algorithm}

Intuitively, we maintain both our estimate of \( \mu \) in the form of \( \mu_t \), \emph{and} our ``confidence''---not to be confused with the definition of confidence in the Bayesian setting---about that estimate.
This ``confidence'' is essentially an error bar around \( \mu_t \), the standard error around the sample mean.
If our estimate of an action is high and the error bar is wide, we will therefore tend to pick that action (exploration).
As \( t \) increases, the values of \( N_t \) will increase, and hence the error bars will decrease.
We can then be confident that our estimate \( \mu_t \) is very close to \( \mu \) and we can simply pick the best action every time (exploitation).

\begin{theorem} \label{thm:bandits:ucb}
  At each round \( t = 1, \dots, T \), the UCB algorithm has expected regret of
  \begin{equation*}
    \E(R_t)  \leq  \O\left( \sqrt{K t \log T} \right)  \qquad \text{with probability} \geq 1 - \frac{2K}{T^3} .
  \end{equation*}
  Thus, the UCB algorithm has no regret w.h.p.
\end{theorem}

\begin{proof}
  Define \( \varepsilon_t(a) = \sqrt{(2 \log T) / N_t(a)} \).
  Suppose the event \( E = \{ \abs{\hat\mu_t(a) - \mu(a)} \leq \varepsilon_t(a); \forall a \in \A, \forall t = 1,\dots,T \} \) holds.
  Let \( a_* \) and \( a_t \) be the (unknown) optimal arm and the selected arm at time \( t \), respectively.
  Since \( a_t \) is selected at time \( t \), then by the algorithm, \( \mathrm{UCB}_t(a_t) \geq \mathrm{UCB}_t(a_*) \).
  Since \( E \) holds, \( \mu(a_t) + \varepsilon_t(a_t) \geq \hat\mu(a_t) \).
  Moreover, by definition, \( \mathrm{UCB}_t(a_*) \geq \mu(a_*) \).
  Therefore,
  \begin{equation*}
    \mu(a_t) + 2 \varepsilon_t(a_t) \geq \hat\mu(a_t) + \varepsilon_t(a_t) = \mathrm{UCB}_t(a_t) \geq \mathrm{UCB}_t(a_*) \geq \mu(a_*) .
  \end{equation*}
  Rearranging, we have
  \begin{equation*}
    \Delta_t(a_t)  := \mu(a_*) - \mu(a_t) \leq 2 \varepsilon_t(a_t) = 2 \sqrt{(2 \log T) / N_t(a_t)} .
  \end{equation*}
  We will use this bound to obtain the bound for \( \E(R_t) \).

  Since we pick a single action at each time step, first we note that \( t = \sum_{a \in \A} N_t(a) \).
  Moreover, the expected total regret \( \E(R_t) \) can be decomposed over actions:
  \begin{equation*}
    \begin{aligned}
      \E(R_t) & = \sum_{i=1}^t \Delta_t(a_i) = \sum_{a \in \A} \sum_{j=1}^{N_t(a)} \Delta_t(a) \\
              & = \sum_{a \in \A}  2 \sqrt{(2 \log T) / N_t(a)} N_t(a)                         \\
              & = 2 \sqrt{(2 \log T)} \sum_{a \in \A} \sqrt{N_t(a)} .
    \end{aligned}
  \end{equation*}

  Now, notice that \( \sqrt{\, \cdot \,} \) is a concave function.
  By Jensen's inequality, we can then bound the average of \( \sqrt{N_t} \) by (recall that \( \abs{\A} = K \))
  \begin{equation*}
    \frac{1}{K} \sum_{a \in \A} \sqrt{N_t(a)} \leq \sqrt{\frac{1}{K} \sum_{a \in \A} N_t(a)} = \sqrt{\frac{t}{K}} .
  \end{equation*}
  This implies that \( \sum_{a \in \A} \sqrt{N_t(a)} \leq K \sqrt{t/K} = \sqrt{K t} \ \)

  Therefore, we can bound \( \E(R_t) \) by
  \begin{equation*}
    \E(R_t) \leq 2 \sqrt{2} \sqrt{\log T} \sqrt{K t} = \O\left( \sqrt{K t \log T} \right) .
  \end{equation*}
  Taking \( t = T \), we clearly see that \( \E(R_T) \) is sublinear.
  Thus, the UCB algorithm has no regret.

  The last thing we need to show is the probability that the results above hold.
  I.e., we want to show that \( E \) holds with high probability.
  By Hoeffding's inequality and subsituting in \( \varepsilon_t(a) \), we obtain \( \Pr( \abs{\hat\mu_t(a) - \mu(a)} \geq \varepsilon_t(a) ) \leq 2/T^4 \).
  Then, by the union bound over \( a \) and \( t \), we obtain \( \Pr(E^c) \leq (2KT)/T^4 \).
  Therefore, \( \Pr(E) \geq 1 - 2K/T^3 \).
  That is, our analysis below will hold with high probability.

\end{proof}

%% file: sources/005_gp.tex

\chapter{Gaussian Processes}
\label{ch:gp}

Let \( f: \X \to \R \) be a function.
When \( X \) is finite, one can think of \( f \) as a collection of function values \( (f(x))_{x \in \X} \) computed across \defword{evaluation/context points} \( \X \).
The same intuitive image can be useful to think of \( f \) in the infinite case.

A \defword{Gaussian process (GP)} can be seen as a probability distribution on a function space \( \H = \{ f: \X \to \R \} \) \citep{williams2006gaussian}.
The defining property of a GP is that any finite collection of evaluation points \( (x_i)_{i=1}^n \subset \X \), the probability distribution over \( (f(x_i))_{i=1}^n \) is \emph{multivariate Gaussian}.
A GP is fully characterized by its \defword{mean function} \( \mu: \X \to \R \) and its \defword{covariance function} \( k: \X \times \X \to \R \).

The covariance function, expressed through a (positive-definite) \defword{kernel} \( k: \X \times \X \to \R \), with the property that it is symmetric in its two arguments and
\begin{equation} \label{eq:gp:pd-kernel}
  \sum_{i=1}^n \sum_{j=1}^n c_i c_j k(x_i, x_j) \geq 0
\end{equation}
holds for all \( (x_i \in \X)_{i=1}^n \), \( (c_i \in \R)_{i=1}^n \), and \( n \in \mathbb{N} \).
The latter can be expressed through linear algebra:
Let \( (\mK)_{ij} = k(x_i, x_j) \) be the matrix with coefficients equal all evaluations of \( k \) under \( (x_i)_{i=1}^n \).
Then \eqref{eq:gp:pd-kernel} is equivalent to saying that \( \mK \) is positive semi-definite.

An example of commonly-used covariance functions is the \defword{Mat\`{e}rn kernel} with smoothness parameter \( \nu \).
This class of kernels induces a GP over the space of functions that is up to \( k \)-times differentiable for \( k < \nu \).
So, with \( \nu = 5/2 \), the GP is over the space of functions that are twice differentiable.
Another example is the \defword{radial basis function (RBF) kernel}, also known as the \defword{squared exponential kernel}.
This can be seen as the limit of the Mat\'{e}rn kernel when \( \nu \to \infty \).
It thus induces a GP on $C^\infty$.
See standard Gaussian process textbooks, e.g.\ \citet{williams2006gaussian}, for definitions.

\section{Posterior Inference}
\label{ch:gp:posterior}

GPs are useful to make predictions about an unknown function \( f \).
Let \( \D := \{ (x_i, f(x_i)) \}_{i=1}^n \) be a dataset.
Assuming a GP prior\footnote{In practical applications, \( \mu \) is often simply set to the zero function.} \( p(f) = \GP(0, k) \) over \( f \), the GP posterior is described through the updated mean and covariance functions \( \mu(\cdot \mid \D): \X \to \R \) and \( k(\cdot, \cdot \mid \D): \X \times \X \to \R \), respectively.
They are characterized by
\begin{align}
  \label{eq:gp:post-mean}
  \mu(x \mid \D) & = k(x, X) \left( k(X, X) + \sigma_n^2 I \right)^\inv Y                   \\
  \label{eq:gp:post-cov}
  k(x \mid \D)   & = k(x, x) - k(x, X) \left( k(X, X) + \sigma_n^2 I \right)^\inv k(X, x) ,
\end{align}
where we have defined shorthands \( X := (x_i)_{i=1}^n \) and \( Y := (f(x_i))_{i=1}^n \in \R^n \).
Also, \( k(x, X) \), \( k(X, X) \), and \( k(X, x) \) are the matrix representations of the kernel under those evaluation points.
Finally, \( \sigma_n^2 > 0 \) is a measurement noise assumed in evaluating \( f(x) \), i.e.\ \( y = f(x) + \varepsilon \) where \( \varepsilon \sim \N(0, \sigma_n^2) \).

\section{Reproducing Kernel Hilbert Space}
\label{ch:gp:rkhs}

As mentioned before, a GP defines a probability distribution on a function space.
What exactly is that function space?
Inspecting \eqref{eq:gp:post-mean}, we see that \( \GP(0, k) \) describes a set of posterior means
\begin{equation}
  \label{eq:gp:rkhs-mean}
  x \mapsto \sum_{i=1}^n \alpha_i k(x_i, x)  \qquad \text{where } \alpha_i = \left( k(X, X) + \sigma_n^2 I \right)^\inv Y ,
\end{equation}
for under all possible dataset \( \D \).
Note that this set of functions is fully characterized by the choice of the kernel of a GP.
Indeed, \( (k(x_i, \cdot) \in \R^n)_{i=1}^n \), seen as vectors, act as a basis of the resulting functions.
These basis vectors vary depending on the evaluation points \( X \).

We define the \defword{reproducing kernel Hilbert space (RKHS)} \( \H_k \) of \( \GP(0, k) \) to be the completion of the space of functions above.
It is endowed with the inner product
\begin{equation}
  \inner{f, f'} = \sum_{i=1}^n \sum_{j=1}^m \alpha_i \alpha'_j k(x_i, x'_j)
\end{equation}
for \( f = \sum_{i=1}^n \alpha_i k(x_i, \cdot) \) and \( f' = \sum_{j=1}^m \alpha'_j k(x'_j, \cdot) \).

The RKHS inner product induced a norm \( \norm{\cdot}_{H_k} \) that tells us about the ``complexity'' of a function in \( H_k \).
Under this norm, we can define the \defword{RKHS ball} of radius \( r \) by
\begin{equation}
  \H_k[b] := \{ f \in H_k \text{ such that } \norm{f}_{H_k} \leq b \} ,
\end{equation}
which contains all possible GP posterior means under a kernel \( k \) with ``complexity'' at most \( b \).

\section{Information Capacity}
\label{sec:gp:info}

Since GPs are useful for learning an unknown function \( f \) through a dataset \( \D \), it is useful to know how well we can learn \( f \) with a GP prior \( \GP(0, k) \) through noisy observations of \( f \) with noise variance \( \sigma_n^2 \).
This notion is termed \defword{information capacity}.
Intuitively, the information encoded in the GP prior through the covariance function \( k \) determines the information content of \( f \), while the noise level \( \sigma_n^2 \) limits the amount of information provided by observations.

The information regarding \( f \) expressed through \( \D \) can be described by the \defword{mutual information}, also known as the \defword{information gain}:
\begin{equation}
  \mathrm{MI}(Y, f) := \frac{1}{2} \log \det (I + \sigma_n^{-2} K(X, X)) .
\end{equation}
The information capacity is then defined as the maximum information gain through a dataset \( \D = (X, Y) \) of size \( T \):
\begin{equation} \label{eq:gp:info-capacity}
  \gamma_T(f) := \sup_{\abs{\D} = T} \mathrm{MI}(Y, f) .
\end{equation}
As a motivating example, \( \D \) could be obtained through a sequential decision-making process, and we want to know how well we have learned about an unknown function \( f \) under some observation noise \( \sigma_n^2 \) after \( T \) steps.
If the function \( f \) is clear from the context, one can also simply write this quantity as \( \gamma_T \).

For compact \( \gX \subset \R^d \) and a fixed \( \sigma_n \), we have the following, depending on the covariance function \( k \) \citep{srinivas2010gpucb}:
\begin{itemize}
  \item Mat\'{e}rn with smoothness parameter \( \nu \): \( \gamma_t = \O(T^\alpha (\log T)^{1- \alpha}) \) where \( \alpha = d/(2\nu + d) \).
  \item RBF: \( \gamma_T = \O((\log T)^{d+1}) \).
\end{itemize}
See \citet{srinivas2010gpucb} for the detailed discussion.
The intuition is as follows:
The smoother the function \( f \) is (i.e., as \( \nu \) increases), the less information we gain through new data points, since we can already easily predict the function values on the other regions of \( \gX \).
Put another way, smooth functions have less ``surprise''.

The following result is an important application of the maximum information gain.
We will use it extensively in the subsequent chapters.
Suppose we have selected \( T \) observations at context points \( (x_t)_{t=1}^T \).
Let \( (\sigma_t^2(x_t))_{t=1}^T \) be the predictive variance of \( x_t \)'s under the GP at each time step \( t \).
Through the chain rule for mutual information, the information capacity \eqref{eq:gp:info-capacity} can be written as
\begin{equation}
  \mathrm{MI}(Y, f) = \frac{1}{2} \sum_{t=1}^T \log \left( 1 + \frac{\sigma_t^2(x_t)}{\sigma_n^2} \right) .
\end{equation}

\begin{boxedtheorem}[\citeauthor{srinivas2010gpucb}, \citeyear{srinivas2010gpucb}] \label{prop:gp:sum-variance}
  Given \( m \in \R\), let \( k(x, x) \leq m \) for all \( x \in \X \).
  Then \( \sum_{t=1}^T \sigma_t^2(x_t) = \O(\gamma_T) \).
  More specifically, \( \sum_{t=1}^T \sigma_t^2(x_t) \leq \frac{2m}{\log(1 + \sigma^{-2}_n m)} \gamma_T \).
\end{boxedtheorem}

\section{Useful Inequalities}
\label{sec:gp:ineq}

The following result, known as (some variant of) the Borell-TIS inequality \citep{van1996weak}, is useful to bound the frequency of the supremum of GP sample paths.

\begin{boxedtheorem}[Borell-TIS Inequality] \label{thm:borell-tis}
  Let \( \X \) be a topological space and let \( f \sim \GP(0, k) \) be a sample path of a centered Gaussian process on \( \X \).
  If \( \sup_{x \in \X} \abs{f(x)} \) finite, then for every \( \lambda > 0 \),
  \begin{equation}
    \Pr( {\textstyle \sup_{x \in \X} \abs{f(x)} \geq \lambda } ) \leq 2 \exp\left( \frac{-\lambda^2}{8 \E \left( \sup_{x \in \X} \abs{f(x)} \right)^2} \right) .
  \end{equation}
\end{boxedtheorem}

%% file: sources/006_bayesopt_discrete.tex

\chapter{Discrete Bayesian Optimization}
\label{ch:discrete-bo}

To start, we assume the search space (action space in the bandit lingo) \( \X \) is finite.
This is practically very relevant, e.g.\ in drug and materials discovery applications.

In \defword{Bayesian optimization (BO)}, we want to (w.l.o.g.) maximize an \emph{unknown} function \( f: \X \to \R \).\footnote{For simplicity, we assume a real-valued function.}
This implies that the maximizer \( x_* = \argmax_{x \in \X} f(x) \) is also unknown.
While we do not know \( f \) holistically, we assume we can \emph{evaluate} \( f(x) \) for any \( x \in \X \).
Note, however, that this evaluation is, in general, very costly, and we want to find the maximum with as few evaluations as possible.

Since \( f \) is unknown, we define a prior \( p(f) = \GP(\mu, k) \) with a mean function \( \mu: \X \to \R \) and kernel/covariance function \( k: \X \times \X \to \R \).
At each iteration \( t = 1, \dots, T \), a BO algorithm will select an evaluation point\footnote{One can also select a \emph{batch} of evaluation points, but this is outside the scope of the current discussion.} \( x_t \in \X \) through an acquisition function \( \alpha(x; D_t) = \E_{p(f \mid \D_t)}(u(x, f)) \) where \( u \) is a utility function and \( p(f \mid \D_t) \) is the posterior belief over \( f \) after observing previously gathered data points \( \D_t = \{ (x_i, f(x_i)) \}_{i=1}^{t-1} \).
More specifically, the algorithm will select \( x_t = \argmax_{x \in \X} \alpha(x; D_t) \) and evaluate \( f(x_t) \).
This process is repeated until termination at time \( T \); see \cref{alg:discrete-bo}

\begin{algorithm}
  \small

  \caption{Discrete GP-UCB for BO}
  \label{alg:discrete-bo}

  \begin{algorithmic}[1]
    \Require{Time budget \( T \), GP prior \( \GP(\mu, k) \), unknown function \( f \)}
    \Ensure{Maximum of \( f \) found after \( T \) steps}

    \For{\( t = 1, \dots, T \)}
    \State Count \( N_t(a) \) for each \( a \in \gA \)
    \State Compute \( \mu_t(a) \) for each \( a \in \gA \)
    \State \( a_t = \argmax_{a \in \gA} \mu_t(a) + \sqrt{(2 \log T) / N_t(a)} \)
    \State \( r_{ta} = \texttt{do\_action}(a_t) \)
    \State \( r_\text{total} = r_\text{total} + r_{ta} \)
    \EndFor

    \State \Return \( r_\text{total} \)
  \end{algorithmic}
\end{algorithm}

As in the bandit case, we can use \emph{regret} as a measure of BO performance.
First, we define \defword{instantaneous regret}:
\begin{equation}
  r_t := f(x_*) - f(x_t) ,
\end{equation}
i.e., it measures how far away we are from the maximum when we pick a particular evaluation point \( x_t \in \X \).
Then, we define \defword{cumulative regret} by summing:
\begin{equation}
  R_T := \sum_{t=1}^T r_t = \sum_{t=1}^T f(x_*) - f(x_t) = T f(x_*) - \sum_{t=1}^T f(x_t) .
\end{equation}
A BO algorithm is said to have \defword{no regret} if
\begin{equation*}
  \lim_{T \to \infty} R_T/T = 0 .
\end{equation*}
That is, \( R_T \) is \emph{sublinear} in \( T \).
Taking into account all sources of randomness (in our belief about \( f \) and the construction of \( \D_t \)), we define the (Bayesian)
\defword{expected regret} by \( \E(R_T) \).
Correspondingly, an algorithm has no regret if \( \lim_{T \to \infty} \E(\R_T)/T = 0 \).
One can also prove bounds on \( R_T \) (and not on \( \E(R_T) \)) by arguing that they hold with high probability.
In fact, the latter is stronger.

In what follows, we prove some results for various assumptions about the acquisition function \( \alpha \), under the following regularity assumptions:
\begin{enumerate}[(i)]
  \item The target function \( f \) can be sampled from the prior \( \GP(0, k) \).
  \item The marginal variance induced by the kernel is bounded: \( k(x, x) \geq 1 \) for all \( x \in \X \).
  \item The observation noise \( \sigma_n^2 \geq 0 \) does not depend on \( x \) (\emph{homoskedastic}).
\end{enumerate}

\section{GP-UCB: High-Probability Regret Bound}
\label{sec:discrete-bo:ucb}

This section introduces one technique for proving a regret bound, i.e., a high-probability regret bound.
The algorithm and proof are adapted from the seminal work of \citep{srinivas2010gpucb}.

Similar to UCB in the bandit setting, we pick an evaluation point \( x_t \) by maximizing the upper confidence bound.
Since we have a posterior distribution over \( f \), given by the GP posterior \( \GP(\mu(\cdot \mid \D_t), k(\cdot, \cdot \mid \D_t)) \), we use it to construct our confidence bound at time \( t \).
Defining \( \mu_t := \mu(\cdot \mid \D_t) \), \( k_t := k(\cdot, \cdot \mid \D_t) \), and \( \sigma_t(x) := \sqrt{k_t(x, x)} \) as shorthands, we define our decision-making policy:
\begin{equation} \label{eq:discrete-bo:ucb}
  x_t = \argmax_{x \in \X} \mu_t(x) + \beta_t \sigma_t(x) ,
\end{equation}
where \( \beta_t > 0 \) is a time-dependant hyperparameter.\footnote{Large \( \beta_t \) implies more exploration.}

\begin{theorem}[Discrete GP-UCB] \label{thm:discrete-bo:ucb}
  Let \( X \) be a finite set, \( f: \X \to \R \), and \( \delta \in (0, 1) \).
  Assume \( f \sim \GP(0, k) \) is in the sample paths of the GP prior and w.l.o.g., the marginal variance of the GP is bounded \( k(x, x) \geq 1 \) for any \( x \in \X \).
  For all time horizons \( T \geq 1 \), with \( \beta^2_t = 2 \log\left( \nicefrac{t^2 \pi^2 \abs{\X}}{6 \delta} \right) \), the GP-UCB algorithm has regret
  \begin{equation*}
    R_T \leq \O^*\left(\sqrt{T \gamma_T \log \abs{X}}\right)  \qquad \text{with probability } \geq 1 - \delta ,
  \end{equation*}
  where \( \gamma_T \) is the information capacity of the GP \eqref{eq:gp:info-capacity} and \( \O^* \) is \( \O \) with some log-factors supressed.
\end{theorem}

\begin{proof}
  We define the following confidence interval of a function evaluation \( f(x) \) on \( x \):
  \begin{equation}
    C_t(x) := [\underbrace{\mu_t(x) - \beta_t \sigma_t(x)}_{\mathrm{LCB}_t(x)}, \underbrace{\mu_t(x) + \beta_t \sigma_t(x)}_{\mathrm{UCB}_t(x)}] .
  \end{equation}
  Assume that the following event holds:
  \begin{equation*}
    E = \{ f(x) \in C_t(x) \text{ for all } x \in \X \text{ and for all } t \geq 1 \}.
  \end{equation*}
  Fix \( T \geq 1 \) to be the time horizon of the algorithm.
  Note that, \( f(x_*) \in C_t(x_t) \) for all \( t \geq 1 \) under the event \( E \).
  Therefore, since \( f(x_*) \leq \mathrm{UCB}_t(x_t) \) and \( f(x_t) \geq \mathrm{LCB}_t(x_t) \), a bound of the instantaneous regret \( r_t = f(x_*) - f(x_t) \) follows:
  \begin{equation*}
    \begin{aligned}
      r_t & \leq \mathrm{UCB}_t(x_t) - \mathrm{LCB}_t(x_t) = \mu_t(x_t) + \beta_t \sigma_t(x_t) - \mu_t(x_t) + \beta_t \sigma_t(x_t) \\
          & = 2 \beta_t \sigma_t(x_t) .
    \end{aligned}
  \end{equation*}
  Notice that \( \beta_t \) is non-decreasing and thus we can bound it by \( \beta_t \leq \beta_T \).
  Then, summing up the square of the instantaneous regrets yields
  \begin{align*}
    \sum_{t=1}^T r_t^2 & \leq 4 \sum_{t=1}^T \beta_t^2 \sigma_t^2(x_t) \leq 4 \beta_T^2 \sum_{t=1}^T \sigma_t^2(x_t) \leq  \O(\beta_T^2 \gamma_T)
  \end{align*}
  where we have used the bound on the sum of predictive variances w.r.t.\ the information capacity as described in \cref{prop:gp:sum-variance} with \( m = 1 \).

  Recall that the Cauchy-Schwarz inequality states \( \left(\sum_{t=1}^T a_t b_t\right)^2 \leq \left( \sum_{t=1}^T a_t^2 \right) \left( \sum_{t=1}^T b_t^2 \right) \).
  Letting \( a_t = r_t \) and \( b_t = 1 \) for each \( t = 1, \dots, T \) yields
  \begin{equation*}
    R_T^2 \leq T \sum_{t=1}^T r_t^2 \qquad \implies \qquad R_T \leq \O\left(\sqrt{T \beta_T^2 \gamma_T}\right) .
  \end{equation*}
  Substituting in \( \beta_T^2 = 2 \log\left( \frac{T^2 \pi^2 \abs{\X}}{6 \delta} \right) \) we obtain
  \begin{equation*}
    R_T \leq \O \left( \sqrt{ 2 T (\log \abs{\X} + \log(T^2 \pi^2 / (6\delta)) } \right) = \O^* \left( \sqrt{ 2 T \log \abs{\X} } \right) .
  \end{equation*}
  This proves the regret bound.

  The remaining task is to argue that this bound holds with high probability.
  Recall that we assumed that \( E \) holds.
  We need to show that \( \Pr(E) \geq 1 - \delta \).
  Fix \( t \).
  By \cref{thm:gaussian-tail-srinivas}, we have
  \begin{equation*}
    \Pr(\abs{f(x) - \mu_t(x)} \geq \beta_t \sigma_t(x)) \leq \exp(-\beta_t^2 / 2) ,
  \end{equation*}
  Note that this probability is equivalent to \( \Pr(f(x) \not\in C_t(x)) \).
  By the union bound over \( x \), we obtain
  \begin{equation*}
    \Pr(\{ f(x) \not\in C_t(x), \exists x \in \X \}) \leq \abs{\X} \exp(-\beta_t^2 / 2) = \frac{6 \delta}{t^2 \pi^2} .
  \end{equation*}
  Applying the union bound over \( t \), we obtain
  \begin{equation*}
    \Pr(E^c) = \Pr(\{ f(x) \not\in C_t(x), \exists x \in \X, \exists t \geq 1 \}) \leq \sum_{t=1}^\infty \frac{6\delta}{\pi^2 t^2} = \delta \frac{6}{\pi^2} \sum_{t=1}^\infty \frac{1}{t^2} .
  \end{equation*}
  The last series is the Riemann zeta function and readily evaluates to \( \pi^2/6 \).
  Therefore, \( \Pr(E^c) \leq \delta \) and thus \( \Pr(E) \geq 1 - \delta \).
  The proof is now complete.
\end{proof}

\section{GP-TS: Expected Regret Bound}
\label{sec:discrete-bo:thompson}

Unlike the previous section, here, we study a different proof technique: showing a regret bound in expectation.
I.e., instead of proving a high-probability regret bound, we shall prove the expected regret \( \E(R_T) \).
Note that this analysis is weaker than the high-probability analysis one since we only consider the average case, e.g.\ there might be some unexpected cases/outliers that are not taken into account by the expectation.

\defword{Thompson sampling (TS)} is an algorithm where \( x_t = \argmax_{x \in \X} \hat{f} \) where \( \hat{f} \sim p(f \mid \D_t) \).
In other words, \( x_t \sim p(x_* \mid \D_t) \) since the sampling process above is a single-sample Monte-Carlo approximation of \( p(x_* \mid \D_t) = \int \delta(\argmax f) \, p(f \mid \D_t) \, df \), where \( \delta \) is the Dirac delta distribution.
We consider the case where the posterior \( p(f \mid \D_t) \) is a GP posterior \( \GP(\mu_t, k_t) \) and call the algorithm \defword{GP-TS}.
The analysis below is adapted from \citet{russo2014ts}.

\begin{theorem}[Discrete GP-Thompson-Sampling] \label{thm:discrete-thompson}
  Let \( X \) be a finite set and \( f: \X \to \R \).
  Assume \( f \sim \GP(0, k) \) is in the sample paths of the GP prior and w.l.o.g., the marginal variance of the GP is bounded \( k(x, x) \geq 1 \) for any \( x \in \X \).
  For all time horizons \( T \geq 1 \), GP-TS algorithm has expected regret
  \begin{equation*}
    \E(R_T) \leq \O^*\left(\sqrt{T \gamma_T \log \abs{\gX}}\right) ,
  \end{equation*}
  where \( \gamma_T \) is the information capacity of the GP \eqref{eq:gp:info-capacity} and \( \O^* \) is \( \O \) with some log-factors supressed.
\end{theorem}

\begin{proof}
  Fix a \( t \geq 1 \).
  By the algorithm, since \( x_t \sim p(x_* \mid \D_t) \), the chosen context point \( x_t \) and the maximizer \( x_* \) are identically distributed under the current posterior.
  That is, \( p(x_* \mid \D_t) = p(x_t \mid D_t) \).
  Let \( U_t(x, \D_t) \) be \emph{any} upper confidence bound derived from the posterior, i.e., a function with the form \( U_t(x, \D_t) = \mu_t(x) + \beta_t \sigma_t(x) \) for an arbitrary \( \beta_t > 0 \).

  Note that given the dataset \( D_t \), the upper confidence bound \( U_t(x, \D_t) \) is a deterministic function of \( x \).
  E.g., in the case of UCB, \( U_t \) is a deterministic function of \( x \) given the posterior mean and standard deviation under \( \D_t \).
  This implies
  \begin{equation*}
    \E[U_t(x_*, \D_t) \mid \D_t] = \E[U_t(x_t, \D_t) \mid \D_t] .
  \end{equation*}

  By definition of the expected regret, \( \E(R_T) = \sum_{t=1}^T \E(r_t) = \sum_{t=1}^T \E[f(x_*) - f(x_t)] \).
  By the law of total expectation, the summand is:
  \begin{align*}
    \E(r_t) & = \E_{\D_t} \left[ \E(f(x_*) - f(x_t) \mid \D_t) \right]                                                                      \\
            & = \E_{\D_t} [ \E[f(x_*) - f(x_t) \mid \D_t] + \underbrace{\E[U_t(x_t, \D_t) \mid D_t] - \E[U_t(x_*, \D_t) \mid \D_t]}_{= 0} ] \\
            & = \E_{\D_t} [ \E[f(x_*) - f(x_t) + U_t(x_t, \D_t) - U_t(x_*, \D_t) \mid \D_t] ]                                               \\
            & = \E_{\D_t} [ \E[f(x_*) - U_t(x_*, \D_t) \mid \D_t] + \E[U_t(x_t, \D_t) - f(x_t) \mid \D_t] ] ,
    %
  \end{align*}
  Where the inner expectation is w.r.t.\ the posterior \( p(f \mid \D_t) \).
  This implies that
  \begin{align*}
    \E(R_T) & = \sum_{t=1}^T \E_{\D_t} [ \E[f(x_*) - U_t(x_*, \D_t) \mid \D_t) ] + \sum_{t=1}^T \E_{\D_t} [\E[U_t(x_t, \D_t) - f(x_t) \mid \D_t] ].
  \end{align*}
  Our task is to bound these two sums.

  For the first sum, let \( \beta_t \) in \( U_t(x, \D_t) \) be
  \begin{equation*}
    \beta_t = \sqrt{2 \log \frac{(t^2 + 1) \abs{\X}}{\sqrt{2\pi}}} .
  \end{equation*}
  Let \( z_t(x) := f(x) - U_t(x, \D_t) \) for brevity.
  Since at time \( t \), for any \( x \), the function value \( f(x) \) is \( \N(\mu_t(x), \sigma_t^2(x)) \), and since Gaussians are closed under affine transformations,\footnote{If \( z \sim \N(\mu, \sigma^2) \), then \( az + b \sim \N(a\mu + b, a^2 \sigma^2) \) for constants \( a \) and \( b \).} we have that
  \begin{equation*}
    z_t(x) = (f(x) - \mu_t(x) - \beta_t \sigma_t(x)) \sim \N(-\beta_t \sigma_t(x), \sigma_t^2(x)) .
  \end{equation*}
  Notice that the mean is nonpositive.
  So, by \cref{thm:gaussian-tail-nonpositive} and by our choice of \( \beta_t \), we have
  \begin{equation*}
    \E(z_t(x) \, \I(z_t(x) \geq 0) \mid \D_t) = \frac{\sigma_t(x)}{\sqrt{2 \pi}} \exp\left( \frac{-\beta_t}{2} \right) = \frac{\sigma_t(x)}{(t^2 + 1) \abs{\X}} \leq \frac{1}{(t^2 + 1) \abs{\X}} ,
  \end{equation*}
  where the last inequality uses the hypothesis that \( \sigma_t(x) \leq \sqrt{k(x, x)} \leq 1 \).\footnote{Intuitively, posterior inference in GPs reduces the initial uncertainty. Picture: the GP uncertainty is ``clamped'' around an observation point.}
  We only care about the event where \( z_t(x) \geq 0 \) since those nonnegative values are the contributing factors to our upper bound.
  Notice that this bound does not depend on \( \D_t \) and thus taking the expectation w.r.t.\ \( \D_t \) on both sides yields \( \E(z_t(x)  \, \I(z_t(x) \geq 0)) \leq 1/((t^2 + 1) \abs{\X}) \).

  And so, by summing over \( t \), we arrive at:
  \begin{align*}
    \sum_{t=1}^T \E[f(x_*) - U_t(x_*, \D_t)] & \leq \sum_{t=1}^{\infty} \sum_{x \in \X} \E[z_t(x) \, \I(z_t(x) \geq 0)] \\
                                             & \leq \sum_{t=1}^{\infty} \sum_{x \in \X} \frac{1}{(t^2 + 1) \abs{\X}}    \\
                                             & = \sum_{t=1}^{\infty}  \frac{1}{(t^2 + 1)}   .
  \end{align*}
  This series converges to some constant \( C \leq 1 \), and can later be absorbed in the \( \O \)-notation.

  For the second sum, notice that \( U_t(x_t, \D_t) - f(x_t) \) is distributed as \( \N(\beta_t \sigma_t(x), \sigma_t^2(x)) \) using the same argument as before.
  So, under a choice of \( \D_t \), it has the expected value \( \beta_t \sigma_t(x) \).
  Therefore, we obtain:
  \begin{align*}
    \sum_{t=1}^T \E[U_t(x_t, \D_t) - f(x_t)] & = \E_{\D_t} \left( \sum_{t=1}^T \beta_t \sigma_t(x_t) \right)                                                                 \\
                                             & \leq \E_{\D_t} \left( \beta_T \sum_{t=1}^T \sigma_t(x_t) \right)            &  & \text{(\( \beta_t \) nondecreasing)}         \\
                                             & \leq \E_{\D_t} \left( \beta_T \sqrt{T \sum_{t=1}^T \sigma_t^2(x_t)} \right) &  & \text{(Cauchy-Schwarz)}                      \\
                                             & \leq \E_{\D_t} \left( \beta_T \sqrt{T \O(\gamma_T)} \right)                 &  & \text{(\cref{prop:gp:sum-variance})}         \\
                                             & \leq \beta_T \sqrt{T \O(\gamma_T)}                                          &  & \text{(No dependence on \( \D_t \) anymore)} \\
                                             & = \O^*\left(\sqrt{T \gamma_T \log \abs{\X}}\right) .                        &  & \text{(Substituting in \( \beta_T \))}
  \end{align*}
  Altogether, we conclude that \( \E(R_T) \leq C + \O^*(\sqrt{T \gamma_T \log \abs{\X}}) \) and the proof is complete.
\end{proof}

%% file: sources/007_bayesopt_continuous.tex

\chapter{Continuous Bayesian Optimization}
\label{ch:bayesopt}

We focus on UCB but now assume that the domain \( \X \) of the unknown function \( f: \X \to \R \) is \emph{continuous}.
The assumption here is that \( \X \subset [0, m]^d \subset \R^d \) compact and convex.
This assumption is quite practical since normalization/standardization of inputs (and outputs, for that matter) in continuous BO is standard.
The proof strategy here is to obtain a discretization \( \X_t \) of \( \X \) at each time step \( t \).
Then, in conjunction with a Lipschitz continuity assumption on the sample paths of the GP prior, we extend the regret bound on the discrete space into a continuous space with a known bound.

First, we show that the Lipschitz assumption is quite weak---it is applicable to many standard kernels.
Based on the Borell-TIS inequality (\cref{thm:borell-tis}), we have the following proposition.

\begin{proposition}
  Let \( \GP(0, k) \) be a centered GP on a compact \( d \)-dimensional domain \( \X \) with continuously differentiable sample paths \( f \sim \GP(0, k) \).
  If \( L := \max_i \nicefrac{\partial f}{\partial x_i} \), then for all \( \lambda > 0 \),
  \begin{equation*}
    \Pr(L > \lambda) \leq d a \exp\left( -\frac{\lambda^2}{b^2} \right) ,
  \end{equation*}
  for some constants \( a, b > 0 \).
\end{proposition}

\noindent Now we are ready to state and prove the main result \citep{srinivas2010gpucb}.

\begin{theorem}[Continuous GP-UCB]
  Let \( X \subset [0, m]^d \) compact and convex with \( d \in \mathbb{N} \) and \( m > 0 \).
  Assume w.l.o.g.\ that the marginal variance of the GP on \( \X \) is bounded \( k(x, x) \leq 1 \) for any \( x \in \X \).
  If the objective function \( f: \X \to \R \) is Lipschitz continuous with a Lipschitz constant \( L \) and it is in the sample paths of the GP prior, i.e.\ \( f \sim \GP(0, k) \),
  then, for any \( \delta \in (0, 1) \) and for all time horizons \( T \geq 1 \), with
  \begin{equation*}
    \beta_t = \sqrt{2 \log(\nicefrac{2 \pi t^2 (L m d t^2)^d}{6 \delta})} ,
  \end{equation*}
  the GP-UCB algorithm has regret
  \begin{equation*}
    R_T \leq \O^*\left(\sqrt{T \gamma_T d}\right)  \qquad \text{with probability } \geq 1 - \delta ,
  \end{equation*}
  where \( \gamma_T \) is the information capacity of the GP \eqref{eq:gp:info-capacity} and \( \O^* \) is \( \O \) with some log-factors suppressed.
\end{theorem}

\begin{proof}
  Since \( f \) is \( L \)-Lipschitz,
  \begin{equation*}
    \abs{f(x) - f(x')} \leq L \norm{x - x'} \qquad \text{for any } x, x' \in \X .
  \end{equation*}
  For each \( t \), choose a discretization \( \X_t \) of \( \X \) of size \( \abs{\X_t} = \tau_t^d \) so that for all \( x \in \X \),
  \begin{equation*}
    \norm{x - [x]_t} \leq \frac{m d}{\tau_t} ,
  \end{equation*}
  where \( [x]_t := \argmin_{x' \in \X_t} \norm{x - x'} \).
  Note that a regular grid with \( \tau_t \) many uniformly placed points is sufficient.

  Together they imply that for all \( x \in \X \):
  \begin{equation*}
    \abs{f(x) - f([x]_t)} \leq L \norm{x - [x]_t} \leq \frac{L m d}{\tau_t} .
  \end{equation*}
  By choosing \( \tau_t = L m d t^2 \), i.e.\ by choosing \( \abs{\X_t} = (L m d t^2)^d \), we have for all \( x \in \X \) that
  \begin{equation} \label{eq:cont-bo:proof:eq1}
    \abs{f(x) - f([x]_t)} \leq \frac{1}{t^2} .
  \end{equation}

  Let \( x_* \in \X \) be the maximizer of \( f \).
  Let \( \D_t = \{ (x_i, y_i) \}_{i=1}^{t-1} \) be the dataset up until \( t-1 \).
  We assume that the following events hold:
  \begin{align*}
    E_1 & = \{ f(x_t) \in C_t(x_t) \text{ for all } t \geq 1 \} ,                                               \\
    E_2 & = \{ f(\hat{x}) \in C_t(\hat{x}) \text{ for all } \hat{x} \in \X_t \text{ and for all } t \geq 1 \} ,
  \end{align*}
  where \( C_t(\cdot) = [\mu_t(\cdot) - \beta_t \sigma_t(\cdot), \mu_t(\cdot) + \beta_t \sigma_t(\cdot)] \) is the confidence interval under the GP posterior w.r.t.\ \( \D_t \).

  The event \( E_2 \) implies that for all \( \hat{x} \in \X_t \), we have that \( f(\hat{x}) \leq \mu_t(\hat{x}) + \sqrt{\beta_t} \sigma_t(\hat{x}) \).
  Combining this with \eqref{eq:cont-bo:proof:eq1}, we have that: (Notice that \( [x_*]_t \in \X_t \).)
  \begin{align*}
         & f(x_*) - f([x_*]_t) \leq \frac{1}{t^2}                                   \\
    \iff & f(x_*) \leq f([x_*]_t) + \frac{1}{t^2}                                   \\
    \iff & f(x_*) \leq \mu_t([x_*]_t) + \beta_t \sigma_t([x_*]_t) + \frac{1}{t^2} ,
  \end{align*}
  holds for every \( t \geq 1 \).
  Moreover, if \( x_t \in \X \) is the selected context point at time \( t \), then, by the algorithm and due to the event \( E_1 \), we have that \( \mu_t(x_t) + \beta_t \sigma_t(x_t) \geq \mu_t([x_*]_t) + \beta_t \sigma_t([x_*]_t) \).
  Therefore,
  \begin{equation*}
    f(x_*) \leq \mu_t(x_t) + \beta_t \sigma_t(x_t) + \frac{1}{t^2} .
  \end{equation*}

  We can thus bound the instantaneous regret by:
  \begin{align*}
    r_t & = f(x_*) - f(x_t)                                                \\
        & \leq \mu_t(x_t) + \beta_t \sigma_t(x_t) + \frac{1}{t^2} - f(x_t) \\
        & = \mathrm{UCB}(x_t) - f(x_t) + \frac{1}{t^2}                     \\
        & \leq 2 \beta_t \sigma_t(x_t) + \frac{1}{t^2} .
  \end{align*}
  The last inequality follows since \( f(x_t) \geq \mathrm{LCB}(x_t) \).

  Pick a time horizon \( T \geq 1 \).
  Since \( \beta_t \) is non-decreasing, \( \beta_t \leq \beta_T \) for \( t \leq T \).
  Therefore, as in the discrete case, we obtain
  \begin{equation*}
    \sum_{t=1}^T \left( 2 \beta_t \sigma_t(x_t) \right)^2 \leq 4 \sum_{t=1}^T \beta_t^2 \sigma_t^2(x_t) \leq 4 \beta_T^2 \sum_{t=1}^T \sigma_t^2(x_t) \leq  \O(\beta_T^2 \gamma_T) ,
  \end{equation*}
  where we have used \cref{prop:gp:sum-variance} to bound the sum of the predictive variances.
  By the Cauchy-Schwarz inequality, we obtain
  \begin{equation*}
    \left( \sum_{t=1}^T 2 \beta_t \sigma_t(x_t) \right)^2 \leq T \sum_{t=1}^T \left( 2 \beta_t \sigma_t(x_t) \right)^2  .
  \end{equation*}
  Therefore,
  \begin{align*}
    \sum_{t=1}^T r_t & \leq \sum_{t=1}^T 2 \beta_t \sigma_t(x_t) + \sum_{t=1}^T \frac{1}{t^2}       \\
                     & \leq \O\left(\sqrt{T \beta_T^2 \gamma_T}\right) + \sum_{t=1}^T \frac{1}{t^2} \\
                     & \leq \O\left(\sqrt{T \beta_T^2 \gamma_T}\right) + \frac{\pi}{6}              \\
                     & = \O\left(\sqrt{T \beta_T^2 \gamma_T}\right) ,
  \end{align*}
  where the last inequality follows from the Riemann zeta function \( \sum_{t=1}^\infty \nicefrac{1}{t^2} = \nicefrac{\pi}{6} \).
  By substituting \( \beta_t \) from the hypothesis into the above inequality, we obtain the desired regret bound.

  The remaining task is to bound the probability of the event \( E = E_1 \cap E_2 \), which we have assumed when we derived the regret bound above.
  First, we check each event \( E_1 \) and \( E_2 \) individually.

  \paragraph*{Event \( \bm{E_1} \)}
  For \( E_1 \), notice that \( \beta_t^2 = 2 \log(\nicefrac{2 \pi t^2 (L m d t^2)^d}{6 \delta}) \geq 2 \log(\nicefrac{2 \pi t^2}{6 \delta}) \) since \( (L m d t^2)^d \) is positive and \( \log \) is increasing.
  Then, by \cref{thm:gaussian-tail-srinivas}, we note that for each \( t \geq 1 \),
  \begin{equation*}
    \Pr(\abs{f(x_t) - \mu_t(x_t)} \geq \beta_t \sigma(x_t)) \leq \exp(-\beta_t^2/2) \leq \frac{6 \delta}{2 \pi t^2}.
  \end{equation*}
  Then, through the union bound over \( t \geq 1 \), we have
  \begin{equation*}
    \Pr(E_1^c) \leq \frac{6 \delta}{2 \pi} \sum_{t=1}^{\infty} \frac{1}{t^2} = \frac{\delta}{2} .
  \end{equation*}

  \paragraph*{Event \( \bm{E_2} \)}
  Meanwhile, for \( E_2 \), notice that \( \beta_t^2 = 2 \log(\nicefrac{2 \pi t^2 \abs{\X_t}}{6 \delta}) \) since we have chosen \( \abs{\X_t} = (L m d t^2)^d \).
  Then, by \cref{thm:gaussian-tail-srinivas} again, we have that for each \( \hat{x} \in \X_t \) and each \( t \geq 1 \):
  \begin{equation*}
    \Pr(\abs{f(\hat{x}) - \mu_t(\hat{x})} \geq \beta_t \sigma(\hat{x})) \leq \exp(-\beta_t^2/2) \leq \frac{6 \delta}{2 \pi t^2 \abs{\X_t}}.
  \end{equation*}
  So, through the union bound over \( \hat{x} \in \X_t \), the probability is at most \( \nicefrac{6 \delta}{2 \pi t^2} \).
  And then, through the union bound over \( t \geq 1 \), we obtain \( \Pr(E_2^c) \leq \frac{\delta}{2} \), as in the case of \( E_1 \).

  \vspace{1em}

  Altogether, they imply that
  \begin{equation*}
    \Pr(E^c) = \Pr(E_1^c \cup E_2^c) = \Pr(E_1^c) + \Pr(E_2^c) \leq \frac{\delta}{2} + \frac{\delta}{2} = \delta  .
  \end{equation*}
  This implies that \( \Pr(E) \geq 1 - \delta \).
\end{proof}

%% file: main.bib
@inproceedings{srinivas2010gpucb,
  title = {{G}aussian process optimization in the bandit setting: No regret and
           experimental design},
  author = {Srinivas, Niranjan and Krause, Andreas and Kakade, Sham M and Seeger
            , Matthias},
  booktitle = {ICML},
  year = {2010},
}

@article{russo2014ts,
  title = {Learning to optimize via posterior sampling},
  author = {Russo, Daniel and Van Roy, Benjamin},
  journal = {Mathematics of Operations Research},
  volume = {39},
  number = {4},
  year = {2014},
}

@book{van1996weak,
  title = {Weak Convergence and Empirical Processes With Applications to
           Statistics},
  author = {Van Der Vaart, Aad W and Wellner, Jon A and van der Vaart, Aad W and
            Wellner, Jon A},
  year = {1996},
  publisher = {Springer},
}

@book{hacking2006emergence,
  title = {The emergence of probability: A philosophical study of early ideas
           about probability, induction and statistical inference},
  author = {Hacking, Ian},
  year = {2006},
  publisher = {Cambridge University Press},
}

@article{wald1949decision,
  title = {Statistical decision functions},
  author = {Wald, Abraham},
  journal = {The Annals of Mathematical Statistics},
  year = {1949},
}

@book{williams2006gaussian,
  title = {Gaussian processes for machine learning},
  author = {Williams, Christopher KI and Rasmussen, Carl Edward},
  year = {2006},
  publisher = {MIT Press Cambridge},
}

@book{lattimore2020bandit,
  title = {Bandit algorithms},
  author = {Lattimore, Tor and Szepesv{\'a}ri, Csaba},
  year = {2020},
  publisher = {Cambridge University Press},
}

@article{auer2002ucb,
  title = {Using confidence bounds for exploitation-exploration trade-offs},
  author = {Auer, Peter},
  journal = {JMLR},
  volume = {3},
  number = {Nov},
  year = {2002},
}
